\documentclass{article}

 \usepackage[preprint, nonatbib]{neurips_2026}
 \usepackage[round]{natbib}

\usepackage[utf8]{inputenc} % allow utf-8 input
\usepackage[T1]{fontenc}    % use 8-bit T1 fonts
\usepackage{hyperref}       % hyperlinks
\usepackage{url}            % simple URL typesetting
\usepackage{booktabs}       % professional-quality tables
\usepackage{amsmath,amsfonts, amsthm}       % blackboard math symbols
\usepackage{nicefrac}       % compact symbols for 1/2, etc.
\usepackage{microtype}      % microtypography
\usepackage{xcolor}         % colors
\usepackage{multirow}
\usepackage{multicol}
\usepackage{bm}
\usepackage{graphicx}
\usepackage{xspace}
\usepackage{titletoc}

\title{BaLoRA: Bayesian Low-Rank Adaptation of \\ Large Scale Models}

\author{%
  Dario Coscia\thanks{Work done in collaboration with the University of Amsterdam. Correspondence to \texttt{dario.coscia@sissa.it}.} \\
  mathLab, SISSA\\
  University of Amsterdam\\
  \And
  Sindy Löwe \\
  CuspAI\\
  \And
  Max Welling \\
  CuspAI\\
  University of Amsterdam
}

\newcommand{\eq}[1]{Eq.~\eqref{#1}}

\newcommand{\y}{\bm{y}}
\newcommand{\x}{\bm{x}}
\newcommand{\z}{\bm{z}}

\newcommand{\w}[1]{\bm{\omega}_{#1}}
\newcommand{\bW}[1]{\boldsymbol{\theta}_{#1}}
\newcommand{\W}[1]{\theta_{#1}}

\newtheorem{proposition}{Proposition}

\definecolor{blue}{rgb}{0,0.2,0.5}
\definecolor{green}{rgb}{0.1,0.35,0.0}
\definecolor{red}{rgb}{0.5,0.0,0.0}
\definecolor{purple}{rgb}{0.4,0,0.6}
\definecolor{cyan}{rgb}{0.0,0.4,0.3}
\definecolor{orange}{rgb}{0.6,0.4,0.0}
\definecolor{gray}{rgb}{0.3,0.3,0.3}

\hypersetup{%
  colorlinks,
  linkcolor=cyan,
  citecolor=blue,
  urlcolor=blue
}

\begin{document}

\maketitle

\begin{abstract}
Low-Rank Adaptation (LoRA) has become the standard for fine-tuning large pre-trained models at reduced computational cost. However, its low-rank point-estimate updates limit expressiveness, leave a persistent gap relative to full fine-tuning accuracy, and provide no built-in uncertainty quantification, limiting its applicability in settings where reliability matters as much as accuracy. We introduce BaLoRA, a Bayesian extension of LoRA with a novel input-adaptive Bayesian parameterization of LoRA matrices that adds minimal parameters and compute. Surprisingly, not only does the Bayesian extension yield well-calibrated uncertainty estimates, but the adaptive noise injection underlying our approach also significantly improves prediction accuracy, narrowing the gap with full fine-tuning across both natural language reasoning and vision tasks. When applied to band gap prediction in metal-organic frameworks, BaLoRA produces zero-shot test-time uncertainty estimates that correlate more strongly with model error than a trained ensemble of LoRA models, and improve monotonically with compute without sacrificing accuracy.
\end{abstract}

\section{Introduction}
Large pre-trained models have shown remarkable generalization across a wide range of tasks, from natural language processing~\citep{qin2023chatgpt, touvron2023llama} and multi-modal applications~\citep{liu2023visual} to atomistic simulations~\citep{woodfamily,shoghimolecules,kang2023multi}. Adapting these models to downstream tasks typically relies on full fine-tuning (FT), which updates all parameters using task-specific data. However, as model sizes grow, FT becomes impractical due to high computational and memory costs. Even when feasible, deploying multiple fine-tuned instances is challenging, as each retains the full model footprint.
\begin{figure}[t]
    \centering
    \includegraphics[width=\linewidth]{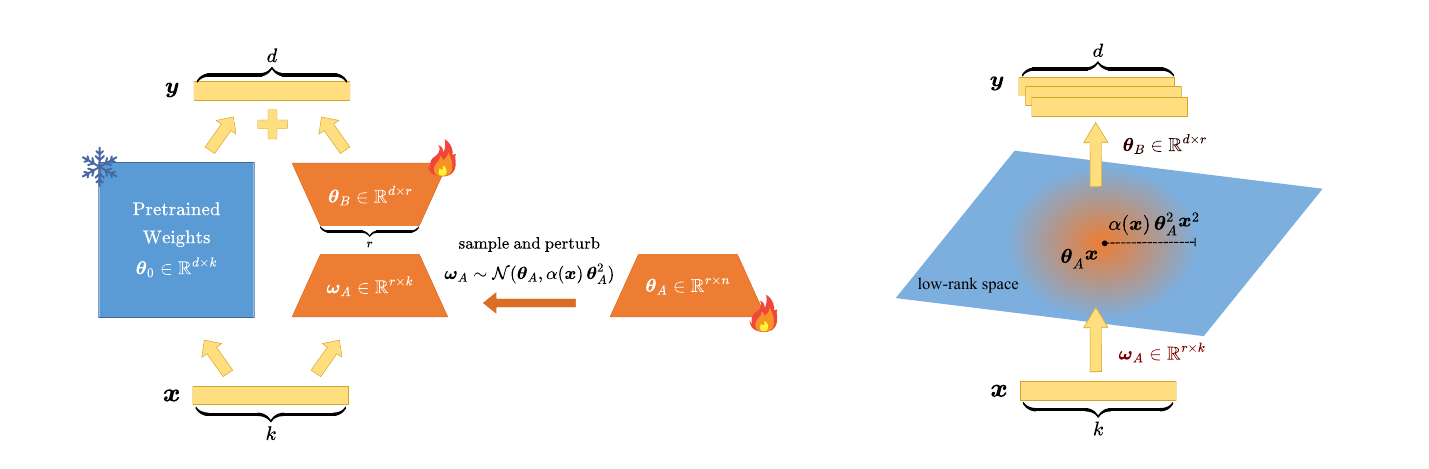}
    \caption{\textbf{Overview of BaLoRA.} \textit{(Left)} BaLoRA extends LoRA by treating the reduction matrix $\bW{A}$ as a random variable with input-adaptive Gaussian noise, $\w{A} \sim \mathcal{N}(\bW{A},\, \alpha(\x)\bW{A}^2)$, while $\bW{0}$ stays frozen. At inference, BaLoRA runs in \textit{deterministic} mode (zero latency, merged adapter) or \textit{stochastic} mode (calibrated uncertainty via sampling). \textit{(Right)} BaLoRA places an input-adaptive uncertainty ellipsoid around the LoRA point estimate, confined to the low-rank space. Uncertainty scales with input activations, concentrating the gradient signal where adaptation is active and acting as an implicit regulariser on the pretrained representations.}
    \label{fig:figure1}
\end{figure}
Parameter-efficient fine-tuning (PEFT)~\citep{houlsby2019parameter} addresses this by updating only a small subset of parameters while freezing the rest, improving efficiency without sacrificing much performance. Low-Rank Adaptation (LoRA)~\citep{hu2022lora} is a widely used PEFT method that constrains updates to a low-rank structure, capturing task-specific changes with far fewer parameters. However, LoRA and related approaches are limited by their low-rank, point-estimate updates, which reduce expressiveness and leave a gap compared to FT~\citep{liu2024dora}. Additionally, they lack uncertainty quantification, which is crucial in safety-critical and data-scarce settings. This is especially important in scientific domains such as materials science and weather forecasting, where reliable uncertainty estimates are essential~\citep{woodfamily, bodnar2025foundation, coscia2025blips,cosciabarnn}.

Drawing on the limitations of existing PEFT methods and the growing need for reliable uncertainty estimates in scientific applications, we present BaLoRA, a Bayesian extension of LoRA that addresses both challenges simultaneously. BaLoRA introduces a novel input-adaptive Bayesian reparametrization of the LoRA reduction matrix, treating its entries as random variables rather than fixed point estimates. This allows uncertainty to be naturally encoded during training via adaptive noise injection, which encourages the reduction matrix to focus on inputs associated with high uncertainty, leading to improved task performance. At inference, BaLoRA can be deployed in deterministic mode by merging adapter weights as in standard LoRA, incurring no additional latency, or in Bayesian mode when uncertainty estimates are required. Our main contributions are as follows:
\begin{itemize}
\item We propose a stochastic PEFT method with an input-adaptive Bayesian reparametrization of the LoRA reduction matrix, where input-dependent noise injection acts as an implicit regulariser that eases training, resulting in performance improvements, while enabling principled uncertainty quantification.
\item We derive a low-rank local reparametrization trick that exploits the LoRA factorization to sample the posterior predictive exactly in the low-dimensional latent space, avoiding materialisation of the full output covariance and matching the computational scaling of standard LoRA.
\item We achieve state-of-the-art commonsense reasoning on 6/8 benchmarks (Llama-3-8B) and image classification (ViT-L/16), and demonstrate the practical value of BaLoRA's uncertainty estimates in a real-world scientific setting (MOF property prediction), where BaLoRA outperforms a trained LoRA ensemble with only a single fine-tuning run.
\end{itemize}

\section{Background and Related Works}
In this section, we present the relevant background on Parameter-Efficient Fine-Tuning, Low-Rank Adapters, Bayesian modelling, and Uncertainty Quantification, which will later support the formulation of our proposed method.

\subsection{Parameter-Efficient Fine-Tuning (PEFT)}
Parameter-efficient fine-tuning (PEFT)~\citep{houlsby2019parameter} adapts large pre-trained models to downstream tasks at low computational cost by updating only a small fraction of parameters. Two main categories can be distinguished. \emph{Prompt-based} methods~\citep{lester2021power, razdaibiedina2023residual, liu2024gpt} prepend learnable soft tokens to the model input, optimizing only these additional vectors while keeping all weights frozen. \emph{Adapter-based} methods~\citep{houlsby2019parameter, mahabadi2021parameter} instead inject small trainable modules into the frozen model, offering strong performance and training stability. However, most adapter-based approaches incur inference overhead, limiting their practical applicability.

Low-Rank Adapters such as LoRA~\citep{hu2022lora} introduce no additional inference latency, built on the observation that fine-tuning updates exhibit low intrinsic rank~\citep{aghajanyan2021intrinsic}. Since LoRA can underperform full fine-tuning across diverse tasks, several extensions have been proposed. DoRA~\citep{liu2024dora} separates weight updates into magnitude and directional components, applying a LoRA-style update only to the directional part. MoRA~\citep{jiang2024mora} projects inputs into a compressed space, applies transformations with a higher-rank matrix, and reconstructs them back to the original space. HiRA~\citep{huang2025hira} leverages a Hadamard product to preserve high-rank update information, enhancing the expressive capacity of the model.
\subsection{Bayesian Models and Uncertainty Quantification}
Bayesian methods account for uncertainty in deep learning by framing learning as probabilistic inference~\citep{hinton1993keeping, welling2011bayesian, graves2011practical, gal2016dropout}, treating weights as random variables updated from prior to posterior via Bayes' theorem~\citep{bayes1763lii}. Since exact posterior computation is intractable at scale, practical approaches rely on approximations. Deep Ensembles combine independently trained models for implicit posterior estimation~\citep{lakshminarayanan2017simple}, extending naturally to parameter-efficient settings such as ensemble LoRA~\citep{wang2023lora}. Alternatively, variational methods cast posterior estimation as an optimization over tractable distributions~\citep{blundell2015weight, kingma2015variational}; notably, Variational Adaptive Dropout~\citep{coscia2025blips} achieves ensemble-comparable performance and uncertainty at the cost of a single network.

\section{Methods}
In this section, we introduce Bayesian Low-Rank Adaptation (BaLoRA), a novel parameter-efficient fine-tuning strategy that builds upon Variational Adaptive Dropout~\citep{coscia2025blips} and extends the LoRA method by naturally encoding uncertainty during both training and inference. The key innovation of BaLoRA lies in the Bayesian reparametrization of the learning process, in which an adaptive variance is incorporated into the LoRA reduction matrix (see Figure~\ref{fig:figure1}). During training, the input-adaptive noise acts as a stochastic regulariser where larger activations in $\w{A}$ induce larger perturbations, leading the model to learn a more robust low-rank representation. Critically the adaptivity ensures perturbations are concentrated where activations are largest, preventing overfitting precisely in the most active directions while leaving stable ones intact. At test time, BaLoRA can be deployed in deterministic inference mode with no additional latency over LoRA, or in Bayesian mode when uncertainty quantification is required.

\subsection{Low Rank Adaptation}
LoRA~\citep{hu2022lora} is a fine-tuning strategy that freezes the pre-trained model weights and injects trainable rank decomposition matrices into each linear layer of a neural network. For a pre-trained weight matrix $\bW{0} \in \mathbb{R}^{k \times d}$, the update is constrained to a low-rank factorization $\bW{B}\bW{A}$, where $\bW{A} \in \mathbb{R}^{r \times d}$ is the \emph{reduction} matrix, and $\bW{B} \in \mathbb{R}^{k \times r}$ is the \emph{reconstruction} matrix, with rank $r \ll \min(d, k)$. Given input $\x \in \mathbb{R}^d$, the output is:
\begin{equation}
    \y = \bW{0}\x + \bW{B}\bW{A}\x.
\end{equation}
Usually, $\bW{A}$ is initialized from an isotropic normal with a small standard deviation, and $\bW{B}$ is initialized to zero, ensuring the adapter contributes no signal at the start of training. Only $\bW{A}$ and $\bW{B}$ are optimized; at inference, the adapter is merged into the pre-trained weights as $\bW{} = \bW{0} + \bW{B}\bW{A}$, incurring no additional inference cost.

\subsection{Bayesian Adaptive LoRA}

To incorporate uncertainty into LoRA, we treat the entries of the reduction matrix as random variables rather than fixed parameters. We indicate with $\w{A}$ the random reduction matrix. To obtain the random reduction matrix, we perturb each entry of the deterministic reduction matrix $\bW{A}$ by input-dependent Gaussian noise as done in Variational Adaptive Dropout~\citep{coscia2025blips}:
\begin{equation}\label{eqn:posterior}
\begin{split}
    \omega_{A;ij} &= \W{A;ij} + \sqrt{\alpha(\x)\W{A;ij}^2} \bm{\epsilon}_{ij}, \quad \bm{\epsilon}_{ij}\sim\mathcal{N}(0,1). \\
    \iff
    q(\w{A}\mid\x)&=\prod_{i=1}^r\prod_{j=1}^d \mathcal{N}(\W{A;ij}, \alpha(\x)\W{A;ij}^2)
\end{split}
\end{equation}
Here, $\alpha(\x)$ is a lightweight inference network that modulates the overall scale of the variance as a function of the input $\x$. Intuitively, weights with larger magnitude carry proportionally more uncertainty, with $\alpha$ controlling how much the input itself drives that uncertainty.

Under this posterior, since $\bW{B}$ is deterministic, the predictive distribution of the output given the input is exactly Gaussian (see Appendix~\ref{sec:predictivedistr} for derivations):
\begin{equation}
    \y \sim \mathcal{N}\Big( \bW{0}\x + \bW{B}\bW{A}\x,\; 
    \alpha\,\bW{B}\,\mathrm{diag}(\bW{A}^2\x^2) \,\bW{B}^\top \Big),
\end{equation}
where squares are applied element-wise. Crucially, the mean coincides with the standard LoRA forward pass, so at inference the weights can still be merged as $\bW{} = \bW{0} + \bW{B}\bW{A}$, incurring no additional inference cost. Additionally, due to the multiplicative noise choice, the only added parameters with respect to LoRA come from the inference network, which is usually light-weight and negligible. 

\paragraph{Training and KL divergence.}
Variational Adaptive Dropout trains the Bayesian model by maximising the Evidence Lower Bound, which trades off between data fit (the expected log-likelihood term) and model complexity (KL divergence between the posterior and the prior). The ELBO is given by:
\begin{equation}\label{eqn:elbo}
\mathcal{L} = \mathbb{E}_{\w{A}\sim q(\w{A}\mid\x)}\left[
\log p(\mathbf{y} \mid \w{A}, \x)\right] - D_{KL}\left[q(\w{A}\mid\x)\mid p(\w{A})
\right].
\end{equation}
In practice, a single Monte Carlo sample from the variational posterior at each training iteration can well approximate \eq{eqn:elbo}. We used the same weakly informed prior of \citet{coscia2025blips}:
\begin{equation}\label{eqn:prior}
    p(\w{A}) = \prod_{i=1}^r\prod_{j=1}^d \mathcal{N}\left(0, \frac{p}{1-p}\W{A;ij}^2\right),
\end{equation}
which can be interpreted as placing dropout to the network with dropout probability $p$. Under this choice, the KL divergence between posterior and prior admits a closed form:
\begin{equation}\label{eqn:kl_loss}
    D_{KL}[q(\w{A}\mid\x)\mid p(\w{A})] = \frac{1}{2}\left(\frac{(\alpha+1)(1+p)}{p}+\log{\frac{p}{1-p}}-\log(\alpha)-1\right),
\end{equation}
which depends only on the scalar $\alpha$ and dropout rate $p$, making it cheap to evaluate and differentiate through. A complete derivation is available in Appendix~\ref{klproof}.

\subsection{Low-Rank local reparametrization trick}
The local reparametrization trick~\citep{kingma2015variational} was introduced to sample the predictive distribution of Bayesian linear layers efficiently. For a single linear layer with input $\x$ and weight matrix $\w{} \sim \mathcal{N}(\bm{\theta}, \alpha \bm{\theta}^2)$, 
the output $\bm{y}$ also follows a normal distribution:
\begin{equation}
\y \sim \mathcal{N}(\bm{\gamma}, \mathrm{diag}(\bm{\delta})),
\end{equation}
where each output element $y_j$ can be written as
\begin{equation}
y_{j} = \gamma_{j} + \sqrt{\delta_{j}}\, \epsilon_{j},\quad\gamma_{j} = \sum_i x_i \theta_{ij}, \quad \delta_{j} = \alpha \sum_i x_i^2 \theta_{ij}^2,  \quad \epsilon_{j} \sim \mathcal{N}(0,1).
\end{equation}
Therefore, rather than sampling the weight matrix $\w{}$ directly, one samples the output $\bm{y}$ from a Gaussian whose mean and variance are computed analytically, reducing both variance and computational cost to two matrix multiplications.

\paragraph{Low-Rank Sampling.}
In BaLoRA, a naive application of the local reparametrization trick is not directly 
viable. Because the stochastic reduction matrix $\bW{A}$ is followed by the deterministic 
projection $\bW{B}$, the posterior predictive covariance of the full LoRA update is $\Sigma=\alpha\,\bW{B}\,\mathrm{diag}(\bW{A}^2\x^2)\,\bW{B}^\top$, which is of size 
$\mathcal{O}(Bk^2)$ for a batch of size $B$ and output dimension $k$. Explicitly forming and sampling from this matrix is computationally prohibitive, even for moderate $k$.

We instead exploit the low-rank structure of the covariance to derive an efficient sampling procedure. Define the element-wise variance vector 
$\mathbf{d}(\x) = \alpha\,\bW{A}^2\x^2 \in \mathbb{R}^r$. The covariance then factors as:
\begin{equation}
    \Sigma = \bW{B}\,\mathrm{diag}(\mathbf{d}(\x))\,\bW{B}^\top = \mathbf{A}(\x)\mathbf{A}(\x)^\top, 
    \quad \mathbf{A}(\x) = \bW{B}\,\mathrm{diag}\!\big(\sqrt{\mathbf{d}(\x)}\big),
\end{equation}
where we used the identity $\mathrm{diag}(\mathbf{d}) = \mathrm{diag}(\sqrt{\mathbf{d}})\,\mathrm{diag}(\sqrt{\mathbf{d}})$. 
Drawing $\boldsymbol{\epsilon}_d \sim \mathcal{N}(\mathbf{0}, \mathbf{I}_r)$, an exact sample from the posterior predictive is then obtained by (see Appendix~\ref{sec:predictivedistr} for derivations):
\begin{equation}
    \y = \bW{0}\x + \bW{B}\bW{A}\x + \bW{B}\big(\sqrt{\mathbf{d}(\x)} \odot \boldsymbol{\epsilon}_d\big),
\end{equation}
where $\odot$ denotes element-wise multiplication. Crucially, the noise is sampled in 
the low-dimensional space $\mathbb{R}^r$ before being projected up by $\bW{B}$, so the full covariance matrix is never materialised. The resulting complexity is $\mathcal{O}(Bkr)$ in time and $\mathcal{O}(Br + Bk)$ in memory, matching the scaling of standard LoRA inference and adding only a constant factor overhead over it. This stands in contrast to $\mathcal{O}(Bk^2)$ for naive sampling of the full covariance, making BaLoRA practical even at large output dimensions where $k \gg r$.

\subsection{Geometric Interpretation and Implicit Regularisation of BaLoRA}
It is instructive to interpret BaLoRA geometrically. Standard LoRA constrains the weight update to a low-dimensional subspace of dimension $r \ll k$: the update $\bW{B}\bW{A}\x$ first projects the input onto the $r$-dimensional row space of $\bW{A}$, and then lifts the result back to the $k$-dimensional output space via $\bW{B}$. The expressivity of the adaptation is therefore deliberately restricted to the column space of $\bW{B}$, denoted $\mathcal{V} = \mathrm{col}(\bW{B}) \subset \mathbb{R}^k$.

BaLoRA preserves this structure while equipping it with an uncertainty estimate. The noise term $\bW{B}\big(\sqrt{\mathbf{d}(\x)} \odot \boldsymbol{\epsilon}_d\big)$ defines a random perturbation that, by construction, also lies in $\mathcal{V}$. Geometrically, the predictive distribution can be understood as placing an \emph{input-adaptive ellipsoid} around the LoRA point estimate, confined entirely to the subspace $\mathcal{V}$:
\begin{equation}
    \y = \underbrace{\bW{0}\x}_{\text{pretrained}} 
    + \underbrace{\bW{B}\bW{A}\x}_{\text{LoRA update}} 
    + \underbrace{\bW{B}\big(\sqrt{\mathbf{d}(\x)} \odot \boldsymbol{\epsilon}_d\big)
    }_{\text{uncertainty in }\mathcal{V}}.
\end{equation}
The shape of this ellipsoid (Figure \ref{fig:figure1} \emph{Left}) is governed by the diagonal matrix $\mathrm{diag}(\mathbf{d}(\x)) = \alpha(\x)\,\mathrm{diag}(\bW{A}^2\x^2)$, which has two geometrically meaningful properties. First, the ellipsoid is \emph{aligned with the LoRA basis}: its principal axes correspond to the columns of $\bW{B}$, so uncertainty is expressed in the same directions as the adaptation itself. Second, the ellipsoid is \emph{input-adaptive}: its extent along each axis scales with $\mathbf{d}(\x)$, meaning that inputs which produce large activations in the reduction matrix $\bW{A}$ are assigned wider uncertainty, while inputs that lie in the null space of $\bW{A}$ carry no uncertainty.

We postulate that this geometric structure eases learning in two ways. First, by confining stochastic perturbations to $\mathcal{V}$, noise is prevented from propagating into the pretrained representations outside the adapted subspace, preserving the stability of directions the model has not chosen to adapt. Second, within $\mathcal{V}$, the input-adaptive scaling has a multiplicative noise form, which induces larger perturbations precisely where activations are strongest, discouraging overconfident updates in the most active directions and therefore acting as an implicit regulariser. Together, these effects allow $\alpha(\x)$ to focus on identifying genuinely ambiguous inputs within the adapted subspace.

\section{Experiments}
To evaluate BaLoRA, we design experiments spanning three domains: natural language, computer vision, and materials science. We first demonstrate that introducing a Bayesian framework enhances LoRA's optimization behavior under fixed hyperparameter settings, assessed on commonsense reasoning benchmarks. On these tasks, we also compare BaLoRA against a range of Parameter-Efficient Fine-Tuning baselines using LLaMA2-7B and LLaMA3-8B, achieving leading performance across multiple settings, alongside an ablation study examining the role of the prior probability. We subsequently broaden our scope to computer vision, evaluating BaLoRA against alternative PEFT methods when adapting a Vision Transformer to several image classification benchmarks. We also examine BaLoRA's uncertainty calibration and carry out a runtime and computational overhead comparison with other PEFT approaches. Finally, we turn to materials science, applying BaLoRA to MOF property prediction. Here, BaLoRA not only surpasses standard LoRA in predictive accuracy but also yields superior uncertainty estimates compared to a LoRA ensemble. Experimental details and hyperparameters can be found in Appendix~\ref{appendix:exp}.

\subsection{Commonsense Reasoning}
We benchmark against prompt-based approaches, specifically Prompt Tuning~\citep{lester2021power} and P-tuning~\citep{liu2022p}, as well as low-rank adaptation methods, namely LoRA~\citep{hu2022lora}, DoRA~\citep{liu2024dora}, MORA~\citep{jiang2024mora}, and HiRA~\citep{huang2025hira}. All experiments are conducted on two open-source Large Language Models (LLMs): Llama-2-7B~\citep{touvron2023llama} and Llama-3-8B~\citep{grattafiori2024llama}. The commonsense reasoning evaluation covers 8 sub-tasks, each with its own predefined train/test split. Unless noted otherwise, we adopt the experimental setup and hyperparameter configuration of~\citet{huang2025hira}. We evaluate commonsense reasoning following the evaluation protocol of \citet{hu2023llm, liu2024dora}, where accuracy is used as the performance metric. For each query, the model's response is determined by identifying the first occurrence of a task-specific keyword (e.g., ``true'' or ``false'' for BoolQ); if no such keyword is found, the response is marked incorrect. 
\paragraph{BaLoRA Training Dynamics Improve LoRA performance}
To isolate the effect of our Bayesian components, we conduct a controlled comparison between LoRA and BaLoRA under identical settings. We set the same weight initialization, dropout probability, and hyperparameters, therefore differing only in the stochastic adaptive injection and KL regularization introduced by BaLoRA. As shown in Table~\ref{tab:peft_comparison_lora}, BaLoRA consistently outperforms LoRA by an average margin of $\sim2.9\%$ on Llama-2 and $\sim5\%$ on Llama-3, supporting our hypothesis that exploiting geometric structure facilitates optimization.
\begin{table}[ttb]
\centering
\caption{Accuracy comparison between standard LoRA and BaLoRA using equivalent dropout coefficients. For LoRA, dropout is applied directly to adapter weights; for BaLoRA, the dropout coefficient serves as $p$ in the KL divergence loss formulation (equation~\eqref{eqn:kl_loss}). \textcolor{blue}{Blue} indicates percentage improvement over LoRA, \textcolor{red}{red} indicates percentage decrease.}
\label{tab:peft_comparison_lora}
\resizebox{\textwidth}{!}{%
\begin{tabular}{llccccccccc}
\toprule
\textbf{Model} & \textbf{Method} &  \textbf{BoolQ} & \textbf{PIQA} & \textbf{SIQA} & \textbf{ARC-c} & \textbf{ARC-e} & \textbf{OBQA} & \textbf{HellaS} & \textbf{WinoG}\\
\midrule
\multirow{2}{*}{Llama-2-7B} & LoRA  & 69.80 & 79.90 & \textbf{79.50} & 64.70 & 79.80 & 81.00 & 83.60 & 82.60\\
& \textbf{BaLoRA} & \textbf{71.59}$^{\textcolor{blue}{+2.6\%}}$ & \textbf{83.73}$^{\textcolor{blue}{+4.8\%}}$ & 77.12$^{\textcolor{red}{-3.0\%}}$ & \textbf{69.45}$^{\textcolor{blue}{+7.3\%}}$ & \textbf{84.39}$^{\textcolor{blue}{+5.8\%}}$ & \textbf{81.40}$^{\textcolor{blue}{+0.5\%}}$ & \textbf{87.73}$^{\textcolor{blue}{+4.9\%}}$ & \textbf{82.72}$^{\textcolor{blue}{+0.1\%}}$ \\
\midrule
\multirow{2}{*}{Llama-3-8B} & LoRA  & 70.80 & 85.20 & 79.90 & 71.20 & 84.20 & 79.00 & 91.70 & 84.30\\
& \textbf{BaLoRA}  & \textbf{75.29}$^{\textcolor{blue}{+6.3\%}}$ & \textbf{88.52}$^{\textcolor{blue}{+3.9\%}}$ & \textbf{80.45}$^{\textcolor{blue}{+0.7\%}}$ & \textbf{78.58}$^{\textcolor{blue}{+10.4\%}}$ & \textbf{89.73}$^{\textcolor{blue}{+6.6\%}}$ & \textbf{84.20}$^{\textcolor{blue}{+6.6\%}}$ & \textbf{95.46}$^{\textcolor{blue}{+4.1\%}}$ & \textbf{85.24}$^{\textcolor{blue}{+1.1\%}}$ \\
\bottomrule
\end{tabular}
}
\end{table}

\paragraph{BaLoRA Achieves State-of-the-Art on Commonsense Reasoning}
Table~\ref{tab:peft_comparison} summarizes accuracy across eight commonsense reasoning benchmarks. BaLoRA ranks first on five out of eight tasks for Llama-2-7B and six out of eight tasks for Llama-3-8B, while remaining within a comparable parameter budget to all low-rank baselines.
The additional traiable parameter count (\emph{Params} in Table~\ref{tab:peft_comparison}) is attributed to the inference network predicting $\alpha$, and is negligable compared to the baseline models. Against the closest competitor, HiRA, BaLoRA delivers steady gains across the majority of sub-tasks, suggesting that the Bayesian geometric prior provides an advantage over purely deterministic rank decomposition. Taken together, these results establish BaLoRA as a robust and parameter-efficient alternative to existing PEFT methods. To the best of our knowledge, we are the first method leading improvements driven by principled uncertainty modeling rather than increased model capacity.

\begin{table}[ttb]
\centering
\caption{Accuracy comparison among various PEFT methods on commonsense reasoning datasets. Results for ChatGPT, LoRA, and DoRA are sourced from \cite{huang2025hira}. The best performance within each LLM is indicated in \textbf{bold}, while the second best performance is highlighted in \underline{underline}.}
\label{tab:peft_comparison}
\resizebox{\textwidth}{!}{%
\begin{tabular}{llcccccccccc}
\toprule
\textbf{Model} & \textbf{Method} & \textbf{Params} (\%) & \textbf{BoolQ} & \textbf{PIQA} & \textbf{SIQA} & \textbf{ARC-c} & \textbf{ARC-e} & \textbf{OBQA} & \textbf{HellaS} & \textbf{WinoG}\\
\midrule
ChatGPT & - & - & 73.10 & 85.40 & 68.50 & 79.90 & 89.80 & 74.80 & 78.50 & 66.10 \\
\midrule
\multirow{7}{*}{Llama-2-7B} & Prompt Tuning & 0.0012 & 55.93 & 12.35 & 30.50 & 6.06 & 8.63 & 9.40 & 6.91 & 40.57\\
& P-Tuning & 0.7428 & 58.75 & 36.02 & 0.20 & 0.17 & 1.98 & 0.80 & 0.01 & 0.00 \\
& LoRA & 0.8256 & 69.80 & 79.90 & 79.50 & 64.70 & 79.80 & 81.00 & 83.60 & 82.60\\
& DoRA & 0.8256 & 71.80 & \underline{83.70} & 76.00 & 68.20 & 83.70 & 82.40 & \underline{89.10} & 82.60 \\
& MoRA & 0.8241 & \underline{72.17} & 80.79 & \underline{79.53} & 71.42 & 85.31 & 81.20 & 29.09 & 80.19 \\
& HiRA & 0.8256 & 71.22 & 83.35 & \underline{79.53} & \textbf{73.81} & \textbf{86.74} & \underline{84.60} & 88.12 & \textbf{83.98} \\
& \textbf{BaLoRA} & 0.8427 & \textbf{72.69} & \textbf{84.93} & \textbf{80.45} & \underline{72.10} & \underline{85.56} & \textbf{85.40} & \textbf{91.23} & \underline{83.66} \\
\midrule
\multirow{7}{*}{Llama-3-8B} & Prompt Tuning & 0.0010 & 56.85 & 45.05 & 36.13 & 31.57 & 32.74 & 29.20 & 14.01 & 50.12\\
& P-Tuning & 0.6240 & 59.97 & 11.64 & 8.19 & 7.42 & 8.63 & 9.60 & 1.77 & 37.65\\
& LoRA & 0.7002 & 70.80 & 85.20 & 79.90 & 71.20 & 84.20 & 79.00 & 91.70 & 84.30\\
& DoRA & 0.7002 & 74.60 & 89.30 & 79.90 & 80.40 & 90.50 & 85.80 & \underline{95.50} & 85.60 \\
& MoRA & 0.6997 & 74.28 & 87.43 & 80.71 & 79.61 & 91.16 & 85.60 & 43.53 & 86.74 \\
& HiRA & 0.7002 & \underline{75.40} & \underline{89.70} & \underline{81.15} & \textbf{82.90} & \textbf{93.27} & \underline{88.32} & 95.36 & \underline{87.70} \\
& \textbf{BaLoRA} & 0.7145 & \textbf{76.42} & \textbf{89.99} & \textbf{81.78} & \underline{80.90} & \underline{91.20} & \textbf{88.40} & \textbf{96.11} & \textbf{88.40} \\
\bottomrule
\end{tabular}
}
\end{table}

\paragraph{BaLoRA is Robust accross Prior Probabilities}
In Figure~\ref{fig:ablation_p} we examine the sensitivity of BaLoRA to the choice of prior probability $p$ in the KL divergence loss of equation~\eqref{eqn:kl_loss} for Llama3 accross BoolQ, PIQA, SIQA and WinoG datasets (see Appendix ~\ref{appendix:exp} for Llama2 and extra datasets). Across both Llama-2 and Llama-3, performance remains largely stable over the full range of tested values, with larger values of $p$ showing marginal accuracy gains across multiple sub-tasks. The variance in average accuracy across $p$ values is narrow, indicating that BaLoRA is robust to this hyperparameter.

\begin{figure}[tbp]
    \centering
    \includegraphics[width=0.8\linewidth]{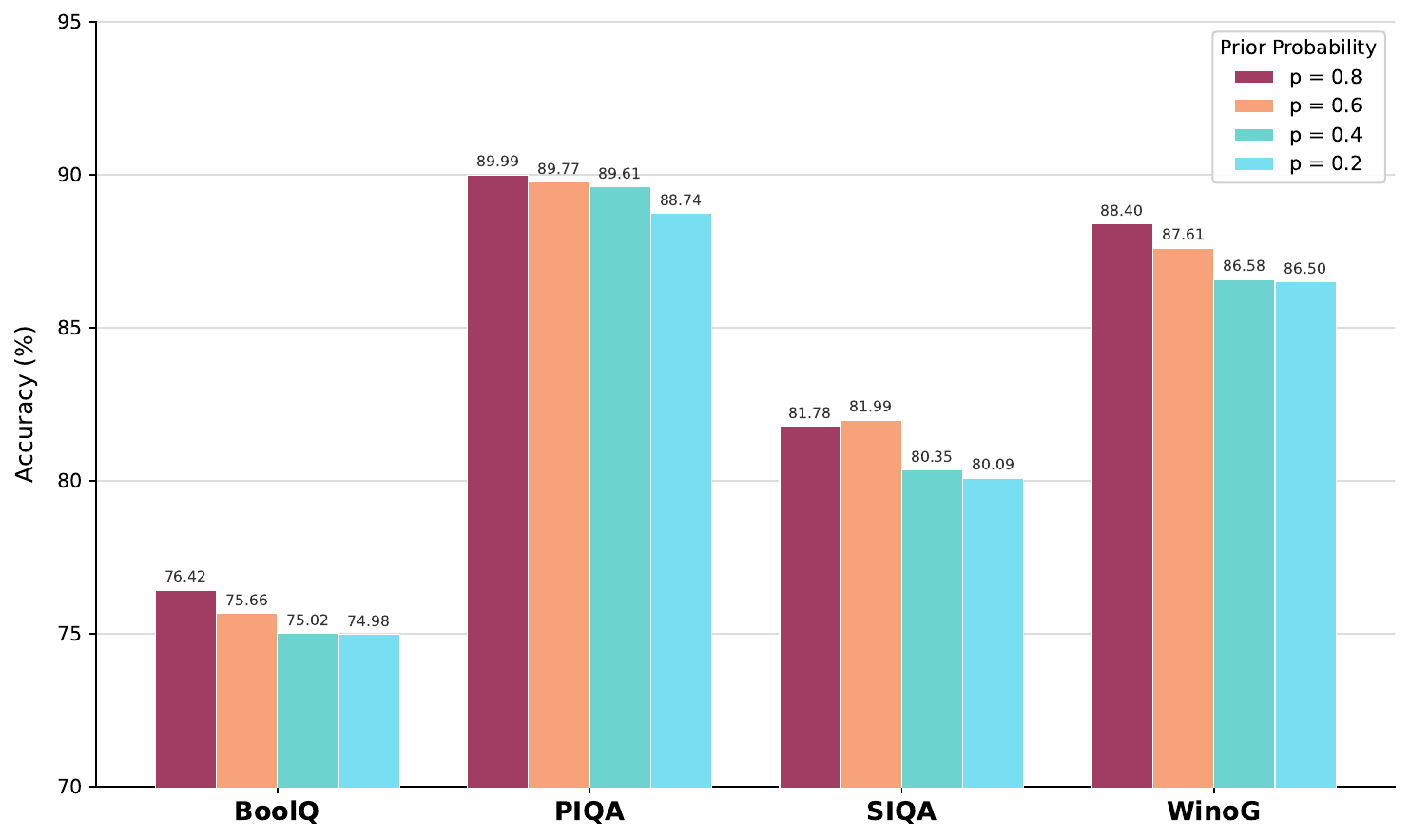}
\caption{Ablation study for LLaMA3 on the effect of prior probability $p$ in the KL divergence loss formulation (equation~\eqref{eqn:kl_loss}) across benchmark tasks. Results demonstrate robustness across different values of $p$.}
    \label{fig:ablation_p}
\end{figure}

\subsection{Visual Perception}

Having established BaLoRA's advantage on language tasks, we turn to visual perception to probe the generality of our approach across domains. We benchmark BaLoRA against LoRA and DoRA, the two most established low-rank adaptation methods in the vision literature, by fine-tuning a ViT-L/16 pretrained on ImageNet-21K~\citep{dosovitskiy2020image} across four image classification benchmarks: CIFAR-10, CIFAR-100, Oxford-IIIT Pets, and Oxford Flowers-102. We follow the standard experimental protocol of~\citet{dosovitskiy2020image} with minor learning rate adjustments, applying adapters exclusively to the query and value projection matrices as is conventional in the \texttt{timm} library. We set $r{=}8$, $\alpha{=}16$ for CIFAR datasets and $r{=}16$, $\alpha{=}32$ otherwise; full hyperparameter details are provided in Appendix~\ref{appendix:additional}. Beyond classification accuracy, we additionally report Expected Calibration Error (ECE), a principled measure of predictive reliability, to assess whether BaLoRA's stochastic training objective yields more calibrated models than its deterministic counterparts.

\paragraph{BaLoRA is Accurate and Calibrated on Image Classification}
\begin{table}[tbp]
\centering
\caption{Accuracy comparison among various PEFT methods on classification datasets for \textbf{ViT-L/16 (ImageNet12K)}. Results for full finetuning are used for reference and are sourced from \citep{dosovitskiy2020image}. The best performance within each ViT is indicated in \textbf{bold}, while the second best performance is highlighted in \underline{underline}.}
\label{tab:peft_comparison_vision_acc}
\begin{tabular}{lccccc}
\toprule
\textbf{Method} &  \textbf{CIFAR10} & \textbf{CIFAR100} & \textbf{Oxford-IIIT Pets} & \textbf{Oxford Flowers-102}\\
\midrule
Full Fine-tuning &  99.15 & 93.25 & 94.67 & 99.61 \\
\midrule
LoRA &  98.93 & \underline{93.59} & 93.96 & 99.44  \\
DoRA & \underline{99.03} & 93.44 & \underline{94.01} & \underline{99.48}  \\
\textbf{BaLoRA} & \textbf{99.04} & \textbf{93.62} & \textbf{94.07} & \textbf{99.61} \\
\bottomrule
\end{tabular}
\end{table}
\begin{table}[tbp]
\centering
\caption{Expected calibration error comparison among various PEFT methods on classification datasets for \textbf{ViT-L/16 (ImageNet12K)}. The best performance within each ViT is indicated in \textbf{bold}, while the second best performance is highlighted in \underline{underline}. Performance is expressed as a percentage.}
\label{tab:peft_comparison_vision_ece}
\begin{tabular}{lccccc}
\toprule
\textbf{Method} &  \textbf{CIFAR10} & \textbf{CIFAR100} & \textbf{Oxford-IIIT Pets} & \textbf{Oxford Flowers-102}\\
\midrule
LoRA &  0.81 & \underline{4.25} & 3.39 & \textbf{0.62}  \\
DoRA & \underline{0.73} & 4.45 & \textbf{3.28} & 0.70  \\
\textbf{BaLoRA} & \textbf{0.71} & \textbf{4.12} & \underline{3.33} & \underline{0.63} \\
\bottomrule
\end{tabular}
\end{table}
We report classification accuracy across four benchmarks in Table~\ref{tab:peft_comparison_vision_acc}. BaLoRA consistently ranks first among all PEFT methods, matching or approaching full fine-tuning performance. Most strikingly, BaLoRA achives state-of-the-art performances for CIFAR100, surpassing even full fine-tuned models; and it closes the gap entirely on Oxford Flowers-102, matching the full fine-tuning ceiling of $99.61\%$ accuracy. Against the strongest PEFT baseline, DoRA, BaLoRA delivers consistent improvements across all four datasets. These results demonstrate that BaLoRA's Bayesian inductive bias is not specific to the language domain but transfers naturally to visual perception, providing a robust and parameter-efficient alternative to deterministic low-rank adaptation.

Beyond predictive accuracy, a well-calibrated model is essential for trustworthy deployment in real-world settings. Table~\ref{tab:peft_comparison_vision_ece} reports the Expected Calibration Error across the four benchmarks. BaLoRA achieves the lowest ECE on two out of four datasets and ranks second on the remaining two, consistently outperforming or matching both LoRA and DoRA without any post-hoc calibration procedure. This suggests that the stochastic adaptive injection and KL regularization jointly act as an implicit regularizer, discouraging overconfident predictions.

\paragraph{BaLoRA adds minimal Computational Cost over LoRA}
Table~\ref{tab:peft_comparison_resources} compares memory and compute overhead across methods during training\footnote{We do not report inference statistics as all models, BaLoRA included, can merge adapters weights with pre-trained weights resulting in zero additional latency/memory cost.}. BaLoRA introduces a modest increase in trainable parameters (1.14M vs. 0.80M for LoRA) and a $\times$1.29 FLOPs overhead relative to full fine-tuning — a direct consequence of the stochastic sampling during the forward pass. Peak memory consumption remains well below full fine-tuning and DoRA, confirming that BaLoRA operates comfortably within the parameter-efficient regime despite its Bayesian formulation.
\begin{table}[htbp]
\centering
\caption{Resources comparison during training among various PEFT methods on CIFAR100 dataset for \textbf{ViT-L/16 (ImageNet12K)}. Results are conducted on a single H100 94Gb for one image input.}
\label{tab:peft_comparison_resources}
\resizebox{\textwidth}{!}{%
\begin{tabular}{lccccccc}
\toprule
\textbf{Method} &  \textbf{Params} (M) & \textbf{Peak Memory Fwd} (Gb) & \textbf{Peak Memory Bwd} (Gb) & \textbf{FLOPs} (G)  \\
\midrule
Full Fine-tuning & 305 & 11.04  & 11.16 & $\times$1  \\
LoRA & 0.80 & 8.13 & 8.15 & $\times$1.003  \\
DoRA & 0.85 & 10.55 & 10.57 & $\times$1.003 \\
\textbf{BaLoRA} & 1.14 & 9.73 & 9.76 & $\times$1.29 \\
\bottomrule
\end{tabular}
}
\end{table}

\subsection{Materials Property Prediction}
After evaluating BaLoRA in language and vision domains, we extend to a scientific setting to assess predictive accuracy and uncertainty quantification, a feature that deterministic low-rank methods lack in regression-based settings. We fine-tune MOFTransformer~\citep{kang2023multi} on the QMOF dataset~\citep{rosen2021machine} for bandgap prediction, an important task in designing MOFs for energy storage applications~\citep{rosen2022high}. Based on preliminary results, we apply adapters only to MLP layers with $r{=}64$ and $\alpha{=}128$. We compare BaLoRA to an ensemble of LoRA models, a strong UQ baseline requiring multiple adapters, while BaLoRA provides calibrated uncertainty from a single adapter. Importantely, training multiple adapters requires multiple fine-tuning runs, while BaLoRA only need one. We evaluate accuracy using Mean Absolute Error (MAE) and uncertainty quality via Spearman correlation between predicted uncertainty and actual error. Full hyperparameter details are provided in Appendix~\ref{appendix:exp}.

\paragraph{BaLoRA is Accurate and Calibrated on Bandgap Prediction}
Figure~\ref{fig:convergence_bandgap} reports MAE and Spearman correlation on the QMOF benchmark. BaLoRA consistently outperforms DoRA and LoRA across both metrics. In terms of accuracy, BaLoRA achieves an MAE of $0.266$eV, narrowing the gap to full fine-tuning ($0.224$eV) while remaining significantly more parameter-efficient (using only $6.75\%$ of trainable parameters). Notably, its predictions more closely match ground truth values in the low bandgap regime ($<2$eV), which is particularly important since lower bandgaps correspond to higher conductivity and are critical for applications such as photocatalysis~\citep{cao2023moformer}.
For uncertainty quantification, BaLoRA achieves the highest Spearman correlation ($0.346$) between predicted uncertainty and actual error, outperforming ensemble DoRA ($0.259$) and LoRA ($0.309$). Interestingly, we find that multiple forward passes at test time further improve uncertainty estimates without requiring additional training (Figure~\ref{fig:convergence_bandgap}). However, all methods exhibit only moderate correlations, underscoring the intrinsic difficulty of reliable uncertainty estimation in this task. This indicates that while BaLoRA improves calibration, bandgap prediction for MOFs remains a challenging setting for uncertainty-aware modeling.

\begin{figure}[tbp]
    \centering
    \includegraphics[width=0.9\linewidth]{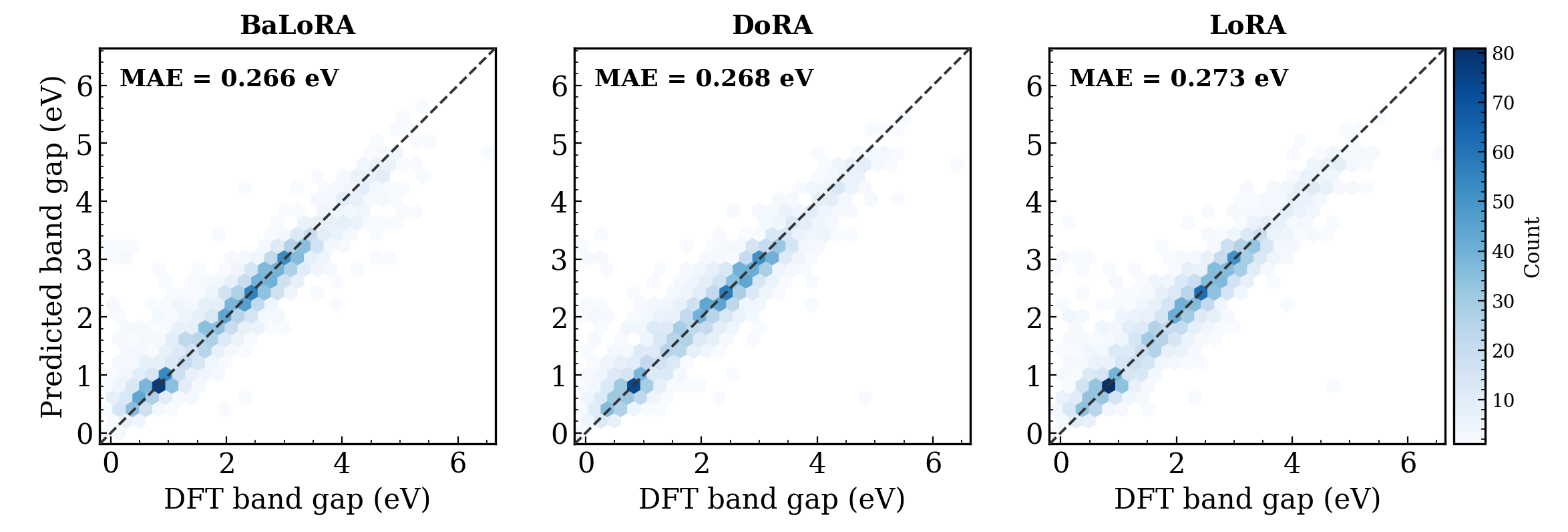}
    \caption{Performance in eV among various PEFT methods on bandgap prediction using \textbf{MOFTransformer} on QMOF dataset.}
    \label{fig:grid_bandgap}
\end{figure}

\section{Conclusions}
We presented BaLoRA, a Bayesian extension of Low-Rank Adaptation that addresses two fundamental limitations of existing PEFT methods: the expressiveness gap relative to full fine-tuning, and the absence of principled uncertainty quantification. By introducing an input-adaptive Bayesian reparametrization of the LoRA reduction matrix, BaLoRA encodes uncertainty during training via adaptive noise injection, which simultaneously acts as an implicit regularizer and enables calibrated predictions at test time. A low-rank local reparametrization trick ensures that the posterior distribution is sampled exactly in the low-dimensional latent space, matching the computational scaling of standard LoRA and incurring no additional inference overhead in deterministic mode.
Across commonsense reasoning, visual perception, and materials property prediction, BaLoRA consistently outperforms LoRA and its deterministic variants while remaining within a comparable parameter and FLOPs budget. On commonsense reasoning, BaLoRA achieves state-of-the-art accuracy on the majority of sub-tasks for both Llama-2-7B and Llama-3-8B, driven by principled uncertainty training rather than increased model capacity. On image classification, it matches full fine-tuning accuracy on Oxford Flowers-102 and achieves strong calibration across all benchmarks without post-hoc correction. Finally, on bandgap prediction for metal-organic frameworks, BaLoRA produces zero-shot uncertainty estimates that correlate more strongly with model error than a trained ensemble of LoRA models, while requiring only a single fine-tuning run.
These results suggest that Bayesian inductive biases, when carefully integrated into the geometry of low-rank adapters, offer complementary and additive advantages over purely deterministic approaches. We hope BaLoRA encourages broader adoption of uncertainty-aware PEFT in scientific and safety-critical applications, where reliable predictions matter as much as accurate ones.

% \begin{ack}
% D. Coscia acknowledges the support provided by PRIN "FaReX - Full and Reduced order modeling of coupled systems: focus on non-matching methods and automatic learning" project
% Acknowledge also NOW cluster
% \end{ack}
\newpage
\bibliographystyle{abbrvnat}
\bibliography{biblio}
%%%%%%%%%%%%%%%%%%%%%%%%%%%%%%%%%%%%%%%%%%%%%%%%%%%%%%%%%%%%

%%%%%%%%%%%%%%%%%%%%%%%%%%%%%%%%%%%%%%%%%%%%%%%%%%%%%%%%%%%%
\newpage
\onecolumn
\appendix
\hrule height 4pt
\vskip 0.15in
\vskip -\parskip
\begin{center}
{\Large\sc Appendix \par}
\end{center}
\vskip 0.29in
\vskip -\parskip
\hrule height 1pt
\vskip 0.4in

\startcontents[sections]\vbox{\sc\Large Table of Contents}\vspace{4mm}\hrule height .5pt
\printcontents[sections]{l}{1}{\setcounter{tocdepth}{2}}
\vskip 4mm
\hrule height .5pt
\vskip 10mm
\newpage

\section{Proofs and Derivations}
This Appendix Section is devoted to the formal proofs introduced in the main text. 
\subsection{Predictive Distribution and Uncertainty Estimates}\label{sec:predictivedistr} 
We begin by deriving the posterior predictive distribution of the output induced by the stochastic LoRA update (Figure~\ref{fig:figure1}).
\begin{proposition}[Predictive Distribution]
Under the posterior
\begin{equation}
q(\w{A}\mid\x)=\prod_{i=1}^r\prod_{j=1}^d 
\mathcal{N}(\W{A;ij}, \alpha(\x)\W{A;ij}^2),
\end{equation}
and with deterministic $\bW{B}$, the output
\begin{equation}
\y = \bW{0}\x + \bW{B}\w{A}\x
\end{equation}
is Gaussian with
\begin{equation}
\y \sim \mathcal{N}\Big(
\bW{0}\x + \bW{B}\bW{A}\x,\;
\alpha(\x)\,\bW{B}\,\mathrm{diag}(\bW{A}^2 \x^2)\,\bW{B}^\top
\Big).
\end{equation}
\end{proposition}
\begin{proof}
Let $\z = \w{A}\x \in \mathbb{R}^r$, whose components are $z_i = \sum_{j=1}^d \omega_{A;ij}\, x_j$. 
Since the $\omega_{A;ij}$ are independent Gaussian random variables,  each $z_i$ is Gaussian, and thus $\z$ is Gaussian. The moments are computed as follows.

\medskip
\noindent\textbf{Mean.}
Using $\mathbb{E}[\omega_{A;ij}] = \W{A;ij}$,
\begin{equation}
\mathbb{E}[z_i]
= \sum_{j=1}^d \mathbb{E}[\omega_{A;ij}]\,x_j
= \sum_{j=1}^d \W{A;ij}\,x_j
= (\bW{A}\x)_i.
\end{equation}
Hence,
\begin{equation}
\mathbb{E}[\z] = \bW{A}\x.
\end{equation}

\medskip
\noindent\textbf{Covariance.}
For $i \neq k$, independence across rows implies
\begin{equation}
\mathrm{Cov}(z_i, z_k) = 0.
\end{equation}
For the variance of each component,
\begin{equation}
\mathrm{Var}(z_i)
= \sum_{j=1}^d \mathrm{Var}(\omega_{A;ij})\, x_j^2
= \sum_{j=1}^d \alpha(\x)\W{A;ij}^2\, x_j^2.
\end{equation}
Thus,
\begin{equation}
\mathrm{Cov}(\z)
= \alpha(\x)\,\mathrm{diag}\!\Big(
\sum_{j=1}^d \W{A;ij}^2 x_j^2
\Big)_{i=1}^r
= \alpha(\x)\,\mathrm{diag}(\bW{A}^2 \x^2).
\end{equation}

\medskip
\noindent\textbf{Distribution of $\y$.}
Finally,
\begin{equation}
\y = \bW{0}\x + \bW{B}\z,
\end{equation}
which is an affine transformation of a Gaussian vector. Therefore,
\begin{equation}
\mathbb{E}[\y] = \bW{0}\x + \bW{B}\bW{A}\x,
\end{equation}
\begin{equation}
\mathrm{Cov}(\y)
= \bW{B}\,\mathrm{Cov}(\z)\,\bW{B}^\top.
\end{equation}

Substituting $\mathrm{Cov}(\z)$ yields
\begin{equation}
\y \sim \mathcal{N}\Big(
\bW{0}\x + \bW{B}\bW{A}\x,\;
\alpha(\x)\,\bW{B}\,\mathrm{diag}(\bW{A}^2 \x^2)\,\bW{B}^\top
\Big).
\end{equation}
\end{proof}

Given the Gaussian posterior predictive derived above, the remaining challenge is efficient sampling. Direct approaches require constructing and factorizing the full output covariance, which scales quadratically in the output dimension and is prohibitive in practice, as explained in the main text. We next show how the low-rank structure of the LoRA parameterization can be exploited to obtain an exact but computationally efficient sampling procedure.
\begin{proposition}[Low-Rank Sampling]
Let the posterior predictive distribution of $\y$ be
\begin{equation}
\y \sim \mathcal{N}\Big(
\bW{0}\x + \bW{B}\bW{A}\x,\;
\Sigma
\Big),
\end{equation}
with covariance
\begin{equation}
\Sigma = \alpha(\x)\,\bW{B}\,\mathrm{diag}(\bW{A}^2 \x^2)\,\bW{B}^\top.
\end{equation}
Define $\mathbf{d}(\x) = \alpha(\x)\,\bW{A}^2 \x^2 \in \mathbb{R}^r$. Then an exact sample from this distribution can be obtained as
\begin{equation}
\y = \bW{0}\x + \bW{B}\bW{A}\x + \bW{B}\big(\sqrt{\mathbf{d}(\x)} \odot \boldsymbol{\epsilon}_d\big),
\quad \boldsymbol{\epsilon}_d \sim \mathcal{N}(\mathbf{0}, \mathbf{I}_r).
\end{equation}
\end{proposition}

\begin{proof}
Let $\boldsymbol{\epsilon}_d \sim \mathcal{N}(\mathbf{0}, \mathbf{I}_r)$ and define
\begin{equation}
\tilde{\z} = \bW{A}\x + \sqrt{\mathbf{d}(\x)} \odot \boldsymbol{\epsilon}_d.
\end{equation}
Since $\boldsymbol{\epsilon}_d$ is standard Gaussian, $\tilde{\z}$ is Gaussian with
\begin{equation}
\mathbb{E}[\tilde{\z}] = \bW{A}\x,
\end{equation}
\begin{equation}
\mathrm{Cov}(\tilde{\z}) = \mathrm{diag}(\mathbf{d}(\x)).
\end{equation}

Now define
\begin{equation}
\y = \bW{0}\x + \bW{B}\tilde{\z}.
\end{equation}
Since $\y$ is an affine transformation of a Gaussian vector, it is Gaussian with mean
\begin{equation}
\mathbb{E}[\y] = \bW{0}\x + \bW{B}\bW{A}\x,
\end{equation}
and covariance
\begin{equation}
\mathrm{Cov}(\y)
= \bW{B}\,\mathrm{Cov}(\tilde{\z})\,\bW{B}^\top
= \bW{B}\,\mathrm{diag}(\mathbf{d}(\x))\,\bW{B}^\top.
\end{equation}

Substituting $\mathbf{d}(\x) = \alpha(\x)\,\bW{A}^2 \x^2$ yields
\begin{equation}
\mathrm{Cov}(\y)
= \alpha(\x)\,\bW{B}\,\mathrm{diag}(\bW{A}^2 \x^2)\,\bW{B}^\top
= \Sigma.
\end{equation}

Therefore, $\y$ has the desired distribution, and the proposed sampling procedure is exact.
\end{proof}
Having derived an exact posterior predictive distribution and an efficient sampling procedure, we now consider two complementary inference modes derived from the same stochastic LoRA formulation. The first is a \emph{deterministic mode}, which replaces the stochastic reduction matrix with its posterior mean. This corresponds to a maximum a posteriori (MAP)-style approximation of the model and enables fast inference, as the resulting LoRA weights can be directly merged into the base model without requiring sampling. The second is a \emph{sampling mode}, where stochastic forward passes are performed using the posterior predictive distribution, enabling the estimation of predictive uncertainty via Monte Carlo sampling.
\paragraph{Deterministic mode.}
In the deterministic setting, we replace the stochastic LoRA parameters with their posterior means. Specifically, for each LoRA layer we obtain the adapted weights
\begin{equation}
\bW{\text{adapted}} = \bW{0} + \bW{B} \bW{A}.
\end{equation}
This recovers the standard LoRA formulation, where the low-rank update $\bW{B} \bW{A}$ is deterministically added to the base weights $\bW{0}$. Importantly, this form allows the adapted weights to be precomputed and merged into the base model prior to inference, it does not require the inference network $\alpha$, and thus it results in no additional computational overhead at test time.

\paragraph{Sampling mode.}
In the sampling setting, we retain the stochasticity of the reduction matrix. For each LoRA layer, we instead use
\begin{equation}
\w{\text{adapted}} = \bW{0} + \bW{B} \w{A},
\end{equation}
where the output $\w{\text{adapted}}
\x$ is sampled (using Low-Rank Sampling) from the learned predictive distribution. Performing multiple stochastic forward passes $s=1,\dots,S$ yields a set of predictive outputs from which we can estimate the predictive mean and variance via Monte Carlo estimation:
\begin{equation}
\hat{\mathbb{E}}[\y] = \frac{1}{S} \sum_{s=1}^S \y^{(s)}, 
\qquad
\hat{\bm{\sigma}}^2 = \frac{1}{S} \sum_{s=1}^S \big(\y^{(s)} - \hat{\mathbb{E}}[\y]\big)^2.
\end{equation}

Moreover, the predictive uncertainty can be further decomposed into epistemic and aleatoric components via law of total variance.
\begin{equation}
\mathrm{Var}(\y \mid \x)=\underbrace{\mathrm{Var}_{\w{}}\!\left(\mathbb{E}[\y \mid \x, \w{}]\right)}_{\text{epistemic}}
+
\underbrace{\mathbb{E}_{\w{}}\!\left[\mathrm{Var}(\y \mid \x, \w{})\right]}_{\text{aleatoric}}.
\end{equation}
In general, $\x$ denotes the model input (e.g., a token sequence for language models or a feature vector in classification settings), processed by a neural backbone. Stochasticity is introduced through the LoRA weights $\w{A}$. The epistemic component captures uncertainty over the adapted parameters, while the aleatoric component reflects the intrinsic predictive uncertainty of the base model conditioned on fixed weights.
\subsection{KL-divergence proof}\label{klproof}
During training we optimise a variational lower bound to jointly maximise predictive performance and control model complexity.  This is achieved by balancing the expected log-likelihood with a KL regulariser between the variational posterior and the prior. In this section we derive the latter.
\begin{proposition}[Closed-form KL-divergence]
Let the variational posterior \( q(\w{A}\mid\x) \) be as defined in \eqref{eqn:posterior} and the prior \( p(\w{A}) \) be as defined in \eqref{eqn:prior}.
Then the KL divergence admits the closed form
$$
D_{KL}\!\left[q(\w{A}\mid\x)\,\|\,p(\w{A})\right]
= \frac{1}{2}\left(\frac{(\alpha(\x)+1)(1+p)}{p} + \log\frac{p}{1-p} - \log\alpha(\x) - 1\right),
$$
which depends only on $\alpha$ and the dropout (prior) rate $p$ hyperparameter.
\end{proposition}
\begin{proof}
Since both $q(\w{A}\mid\x)$ and $p(\w{A})$ factorise over entries, the KL divergence decomposes as
\begin{equation}
D_{KL}\!\left[q(\w{A}\mid\x)\,\|\,p(\w{A})\right]
= \sum_{i=1}^r \sum_{j=1}^d 
D_{KL}\!\left[q(\w{A;ij}\mid\x)\,\|\,p(\w{A;ij})\right].
\end{equation}

From \eqref{eqn:posterior} and \eqref{eqn:prior}, we identify
\[
q(\w{A;ij}\mid\x)=\mathcal{N}\!\big(\W{A;ij},\,\alpha(\x)\W{A;ij}^2\big), 
\quad
p(\w{A;ij})=\mathcal{N}\!\big(0,\,\tfrac{p}{1-p}\W{A;ij}^2\big).
\]

Using the Gaussian KL formula and substituting $\sigma_q^2=\alpha(\x)\W{A;ij}^2$ and $\sigma_p^2=\frac{p}{1-p}\W{A;ij}^2$, the factor $\W{A;ij}^2$ cancels in all terms, yielding
\begin{equation}
D_{KL}\!\left[q(\w{A;ij}\mid\x)\,\|\,p(\w{A;ij})\right] = \frac{1}{2}\left(
\frac{\alpha(\x)+1}{\frac{p}{1-p}}
-1
+\log\frac{p}{1-p}
-\log \alpha(\x)
\right).
\end{equation}

Summing over all $(i,j)$ gives the final expression in \eqref{eqn:kl_loss}.
\end{proof}
In practice, we use a \emph{normalized} version of this KL term by dividing by the number of parameters $rd$ and KL maximum value, so that the regularisation strength is independent of the network width and the maximum KL value equals one.
\section{Experimental Details and Hyperparameters}\label{appendix:exp}
This Appendix Section is devoted to dataset, metrics  and architecture hyperparameters used in the main text for experiments. All our models have been trained on 94GB H100 Nvidia GPUs. All models are trained on the training dataset and evaluated on the test dataset. The validation dataset (if available) is used to select the best checkpoint otherwise the last training checkpoint is used. 

\subsection{Commonsense Reasoning}
The commonsense reasoning task covers 8 sub-tasks (BoolQ, PIQA, SIQA, ARC-c, ARC-e, OBQA, HellaS, and WinoG), each with its own predefined train/test split. The goal of the task is to correctly answer a given multiple-choice question. The models are trained and evaluated as sequence-based models.

\paragraph{Dataset Generation and Availability.}
We adopt the benchmark protocol of \citet{hu2023llm}, directly using their released dataset construction and evaluation setup without modification. In particular, we use their Commonsense170K fine-tuning data, which was constructed by formatting the training sets of BoolQ, PIQA, SIQA, HellaSwag, WinoGrande, ARC-e, ARC-c, and OBQA with predefined task-specific templates. Each dataset is treated as a separate commonsense reasoning task, and all examples are converted into a unified structured prompt consisting of a task description, followed by the input context and multiple-choice options. The datasets and propt-templates are publically available at: \href{https://github.com/AGI-Edgerunners/LLM-Adapters}{https://github.com/AGI-Edgerunners/LLM-Adapters}.

\paragraph{Metrics.}
The main evaluation metric is accuracy. Accuracy is defined as the proportion of correctly answered queries over the total number of queries in the evaluation set. In practice, model predictions are obtained via autoregressive generation with nucleus and beam decoding. Specifically, we use temperature sampling with temperature $0.1$, top-$p$ sampling with $p=0.75$, top-$k$ filtering with $k=40$, and beam search with $4$ beams, as done in \citet{huang2025hira}. The model generates up to 32 new tokens conditioned on the input prompt. The generated sequence is then post-processed by extracting the model’s response segment and parsing the first occurrence of a task-specific keyword using a deterministic rule-based function. This extracted token is treated as the predicted label $\hat{y}_i$. If no valid keyword is found, the prediction is marked as incorrect. Accuracy is then computed over the full evaluation set as
\begin{equation}
\mathrm{Acc} = \frac{1}{N} \sum_{i=1}^{N} \mathbb{I}[\hat{y}_i = y_i],
\end{equation}
where $N$ is the number of samples, $y_i$ is the ground-truth label, and $\mathbb{I}[\cdot]$ is the indicator function.

\paragraph{Architecture and Hyperparameters.}
We adopt the same experimental setup as \citet{huang2025hira}, training for 3 epochs using the AdamW optimizer~\citep{loshchilov2017decoupled} with a linear learning-rate scheduler and a warmup of 100 gradient steps, an effective batch size of 16, a learning rate of $1\times10^{-4}$ for Llama 3 and $2\times10^{-4}$ for Llama 2, and no weight decay. The pretrained backbone model is loaded in 8-bit quantization to reduce memory consumption while preserving inference quality, and training is performed in mixed precision (bf16). Following \citet{huang2025hira}, we apply LoRA adapters to the attention and MLP projection matrices using rank $32$ and lora alpha of $64$.

Our BaLoRA inference network predicts layer-wise uncertainty coefficients conditioned on the current input sequence. It consists of a frozen transformer encoder (i.e., the pretrained backbone model, either Llama 2 or 3) used solely for feature extraction, followed by a lightweight MLP that maps the final hidden representation of the last token to a set of coefficients over LoRA layers. The forward computation can be written as
\begin{equation}
\mathbf{x}\rightarrow \text{Frozen Transformer}
\rightarrow h_{\text{last}}
\rightarrow \text{MLP}(d \rightarrow 256 \rightarrow 256 \rightarrow L)
\rightarrow \text{Softplus} \rightarrow \boldsymbol{\alpha},
\end{equation}
where $L$ denotes the number of LoRA layers and $\boldsymbol{\alpha} \in \mathbb{R}^{L}$ are the predicted coefficients. The transformer encoder is kept entirely frozen and serves only as a feature extractor for the current input sequence, while all trainable parameters reside in the MLP.

\subsection{Visual Perception}
The visual perception task covers 4 image classification benchmarks (CIFAR-10, CIFAR-100, Oxford-IIIT Pets, and Oxford Flowers-102), each with its own standard train/test split. The goal of the task is to correctly classify a given input image into one of the predefined categories. The models are fine-tuned and evaluated using a Vision Transformer backbone with a linear classification head.

\paragraph{Dataset Generation and Availability.}
We evaluate BaLoRA on four standard image classification benchmarks: CIFAR-10, CIFAR-100, Oxford-IIIT Pets, and Oxford Flowers-102. We follow the standard dataset splits provided by each benchmark, using the official training and test partitions without modification. All models are fine-tuned on images resized to $224 \times 224$ and normalized using ImageNet mean and standard deviation. Data augmentation during training includes random horizontal flipping, random rotation, and color jittering, while evaluation is performed without augmentation.

\paragraph{Metrics.} We report Top-1 classification accuracy and Expected Calibration Error (ECE) as evaluation metrics. Accuracy is computed as the fraction of correctly classified samples over the full test set:
\begin{equation}
\mathrm{Acc} = \frac{1}{N} \sum_{i=1}^{N} \mathbb{I}\big[\hat{y}_i = y_i\big],
\end{equation}
where $N$ is the number of test samples, $y_i$ is the ground-truth label, $\hat{y}_i$ is the predicted label, and $\mathbb{I}[\cdot]$ is the indicator function.

ECE is computed using 15 bins with L1 normalization. Let $B_m$ denote the set of indices in bin $m$, $\mathrm{acc}(B_m)$ the empirical accuracy in that bin, and $\mathrm{conf}(B_m)$ the average predicted confidence. Then:
\begin{equation}
\mathrm{ECE} = \sum_{m=1}^{15} \frac{|B_m|}{N} \left| \mathrm{acc}(B_m) - \mathrm{conf}(B_m) \right|.
\end{equation}

\paragraph{Architecture and Hyperparameters.}
We optimize using AdamW with learning rate $10^{-3}$, no weight decay, batch size $512$, and linear scheduling with 10\% warmup. Gradient clipping is applied with a maximum norm of 1.0. All models are trained in mixed precision (bf16) using 8-bit loading for the backbone encoder. We follow the PEFT protocol of \citet{dosovitskiy2020image} and apply LoRA adapters to the query and value projection matrices of all transformer attention layers, as implemented in the \texttt{timm} library. We use rank $r \in \{8, 16\}$ and scaling factor $\alpha \in \{16, 32\}$ depending on dataset complexity, with no additional modifications to the pretrained backbone.

Our backbone model is a Vision Transformer (ViT-L/16) pretrained on ImageNet-21K~\citep{dosovitskiy2020image}. The model processes images as sequences of $16 \times 16$ patches and produces a sequence of token embeddings of dimension $d$. A classification token (CLS) is prepended to the patch sequence and used for downstream prediction.

The BaLoRA inference network uses a second Vision Transformer encoder as a frozen feature extractor. This encoder shares the same architectural family as the backbone (ViT), but is instantiated independently and loaded in 8-bit precision for memory efficiency. Unlike the backbone, this encoder is never updated and serves only to extract a global image representation from the CLS token. We ablate (see Figure~\ref{fig:vision_ablation}) several encoder choices for this module and select DeiT-Base~\citep{touvron2021training}, distilled on ImageNet-1k, as it performs best on validation. In particular, given an input image $\mathbf{x}$, the inference newtork computes:
\[
\mathbf{x} \rightarrow \text{Frozen ViT Encoder} \rightarrow h_{\text{CLS}} \in \mathbb{R}^d
\rightarrow \text{MLP}(d \rightarrow 256 \rightarrow 256 \rightarrow L)
\rightarrow \text{Softplus} \rightarrow \boldsymbol{\alpha},
\]
where $L$ is the number of LoRA layers. The ViT encoder is kept entirely frozen and serves only as a feature extractor for the current input image, while all trainable parameters reside in the MLP.

\subsection{Materials Property Prediction}
The materials property prediction task targets bandgap regression on the QMOF database. The goal of the task is to predict the DFT-computed electronic bandgap (in eV) of a given metal--organic framework structure. The models are fine-tuned and evaluated using a multimodal MOFTransformer backbone that jointly processes crystal graph and three-dimensional energy grid representations.

\paragraph{Dataset Generation and Availability.}
We evaluate BaLoRA on the QMOF bandgap prediction task using the QMOF database~\citep{rosen2021machine}, a standard regression benchmark for metal--organic framework (MOF) property prediction. The dataset consists of DFT-computed electronic bandgap values (in eV) for crystalline MOF structures. We follow the standard train/validation/test split provided by the benchmark without modification. The dataset is publicly available through the \texttt{moftransformer} package~\citep{kang2023multi}.

\paragraph{Metrics.}
The primary evaluation metric is Mean Absolute Error (MAE), computed between model predictions and ground-truth bandgap values on the test set:
\begin{equation}
\mathrm{MAE} = \frac{1}{N}\sum_{i=1}^{N} \lvert \hat{y}_i - y_i \rvert,
\end{equation}
where $N$ is the number of test samples, $y_i$ is the ground-truth bandgap, and $\hat{y}_i$ is the predicted value. Predictions are denormalized using the training-set mean $\mu = 2.0899$ and standard deviation $\sigma = 1.1295$ before computing the metric. For uncertainty estimation, we additionally assess uncertainty calibration quality. Specifically, we perform $T = 100$ stochastic forward passes at test time, compute the per-sample predictive variance, and report the Spearman rank correlation between the predictive variance and the squared prediction error. A higher correlation indicates that the model's uncertainty estimates are better aligned with its actual errors, reflecting more reliable uncertainty quantification.

\paragraph{Architecture and Hyperparameters.}
The backbone model is MOFTransformer~\citep{kang2023multi}, a multimodal transformer that jointly processes crystal graph and three-dimensional energy grid representations of MOF structures. Following preliminary experiments, we apply LoRA adapters to the MLP projection matrices of all transformer blocks, using rank $r = 64$ and scaling factor $\alpha = 128$. Training is performed for $20$ epochs using the AdamW optimizer~\citep{loshchilov2017decoupled} with learning rate $5 \times 10^{-4}$, no weight decay, batch size $32$, and a linear learning-rate schedule with $5\%$ warmup. Gradient clipping is applied with a maximum norm of $1.0$, and training uses mixed precision (bf16). For LoRA and DoRA baselines, we apply a dropout rate of $0.1$ on the adapter weights; for BaLoRA, dropout is disabled. The regression head and all adapter weights are trained end-to-end using an L1 loss on the normalized targets.

The BaLoRA inference network predicts layer-wise uncertainty coefficients conditioned on the input MOF structure. It reuses the frozen graph and volume embedding modules from the pretrained backbone as feature extractors. Given an input sample, graph token embeddings are obtained from the frozen CGCNN~\citep{xie2018crystal} encoder and aggregated via masked mean pooling over valid atom positions. The resulting graph-level representation is summed with the frozen volume embedding and passed through a lightweight MLP head:
\begin{equation}
\mathbf{x} \rightarrow \text{Frozen Graph Encoder} \rightarrow \bar{h}_{\text{graph}} \oplus h_{\text{vol}}
\rightarrow \text{MLP}(d \rightarrow d/2 \rightarrow L)
\rightarrow \text{Softplus} \rightarrow \boldsymbol{\alpha},
\end{equation}
where $L$ denotes the number of LoRA layers, $\oplus$ denotes element-wise addition, and $\boldsymbol{\alpha} \in \mathbb{R}^{L}$ are the predicted coefficients. The BaLoRA KL divergence term uses a prior dropout probability of $0.1$. The graph and volume embedding parameters are kept entirely frozen; all trainable parameters reside in the MLP head.

\section{Additional Results}\label{appendix:additional}
In this section we present additional ablation studies and convergence analyses that complement the main experimental results. We first examine the sensitivity of BaLoRA to the prior probability hyperparameter on commonsense reasoning, then ablate the choice of frozen encoder in the inference network for visual perception, and finally analyze the convergence of uncertainty estimates as a function of Monte Carlo steps on materials property prediction.

\subsection{Additional Results on Commonsense Reasoning}
Table~\ref{tab:ablation_prior_prob} reports the effect of the prior probability $p$ in the BaLoRA KL divergence loss (Equation~\eqref{eqn:kl_loss}) across all eight commonsense reasoning benchmarks. For both Llama-2-7B and Llama-3-8B, accuracy remains stable across a wide range of prior values ($p \in \{0.2, 0.4, 0.6, 0.8\}$), with no single setting consistently dominating across all tasks. On Llama-2-7B, the best average performance is obtained at $p=0.6$, while on Llama-3-8B, $p=0.8$ yields the highest scores on most benchmarks. Importantly, the spread between the best and worst prior settings is modest in all cases, indicating that BaLoRA is robust to this hyperparameter and does not require expensive tuning of $p$ in practice.

\begin{table}[tbp]
\centering
\caption{Ablation study on the effect of prior probability $p$ in the KL divergence loss formulation (equation~\eqref{eqn:kl_loss}) across benchmark tasks. Results demonstrate robustness across different values of $p$.}
\label{tab:ablation_prior_prob}
\resizebox{\textwidth}{!}{%
\begin{tabular}{lccccccccc}
\toprule
\textbf{Model} & \textbf{Prior Probability} ($p$) & \textbf{BoolQ} & \textbf{PIQA} & \textbf{SIQA} & \textbf{ARC-c} & \textbf{ARC-e} & \textbf{OBQA} & \textbf{HellaS} & \textbf{WinoG} \\
\midrule
\multirow{4}{*}{Llama-2-7B} & 0.8 & 72.08 & 84.17 & 79.99 & 70.90 & 84.89 & 81.00 & 91.44 & 84.29 \\
& 0.6 & 72.69 & 84.93 & 80.45 & 72.10 & 85.56 & 85.40 & 91.23 & 83.66 \\
& 0.4 & 71.13 & 83.68 & 76.31 & 71.50 & 84.76 & 83.80 & 90.55 & 83.74\\
& 0.2 & 73.00 & 83.30 & 80.09 & 68.86 & 83.16 & 82.40 & 90.69 & 82.08 \\
\midrule
\multirow{4}{*}{Llama-3-8B} & 0.8 & 76.42 & 89.99 & 81.78 & 80.89 & 91.20 & 88.40 & 96.11 & 88.40 \\
& 0.6 & 75.66 & 89.77 & 81.99 & 80.89 & 90.66 & 86.20 & 95.90 & 87.61\\
& 0.4 & 75.02 & 89.61 & 80.35 & 79.01 & 90.32 & 86.20 & 95.49 & 86.58 \\
& 0.2 & 74.98 & 88.74 & 80.09 & 79.95 & 90.28 & 85.60 & 95.62 & 86.50 \\
\bottomrule
\end{tabular}
}
\end{table}

\subsection{Additional Results on Vision Perception}
Figure~\ref{fig:vision_ablation} ablates the choice of frozen encoder used in the BaLoRA inference network for the visual perception experiments. We compare four Vision Transformer variants of increasing capacity: ViT-Tiny, ViT-Small, ViT-Medium (DeiT-Base), and ViT-Large (ViT-L/16). Across all four classification benchmarks, both accuracy (top) and ECE (bottom) are remarkably consistent regardless of encoder size. These results suggest that the BaLoRA inference network does not require a large or expensive encoder to produce well-calibrated uncertainty coefficients; even the smallest ViT-Tiny variant yields competitive performance, underscoring the lightweight nature of the approach.

\begin{figure}
    \centering
    \includegraphics[width=0.8\linewidth]{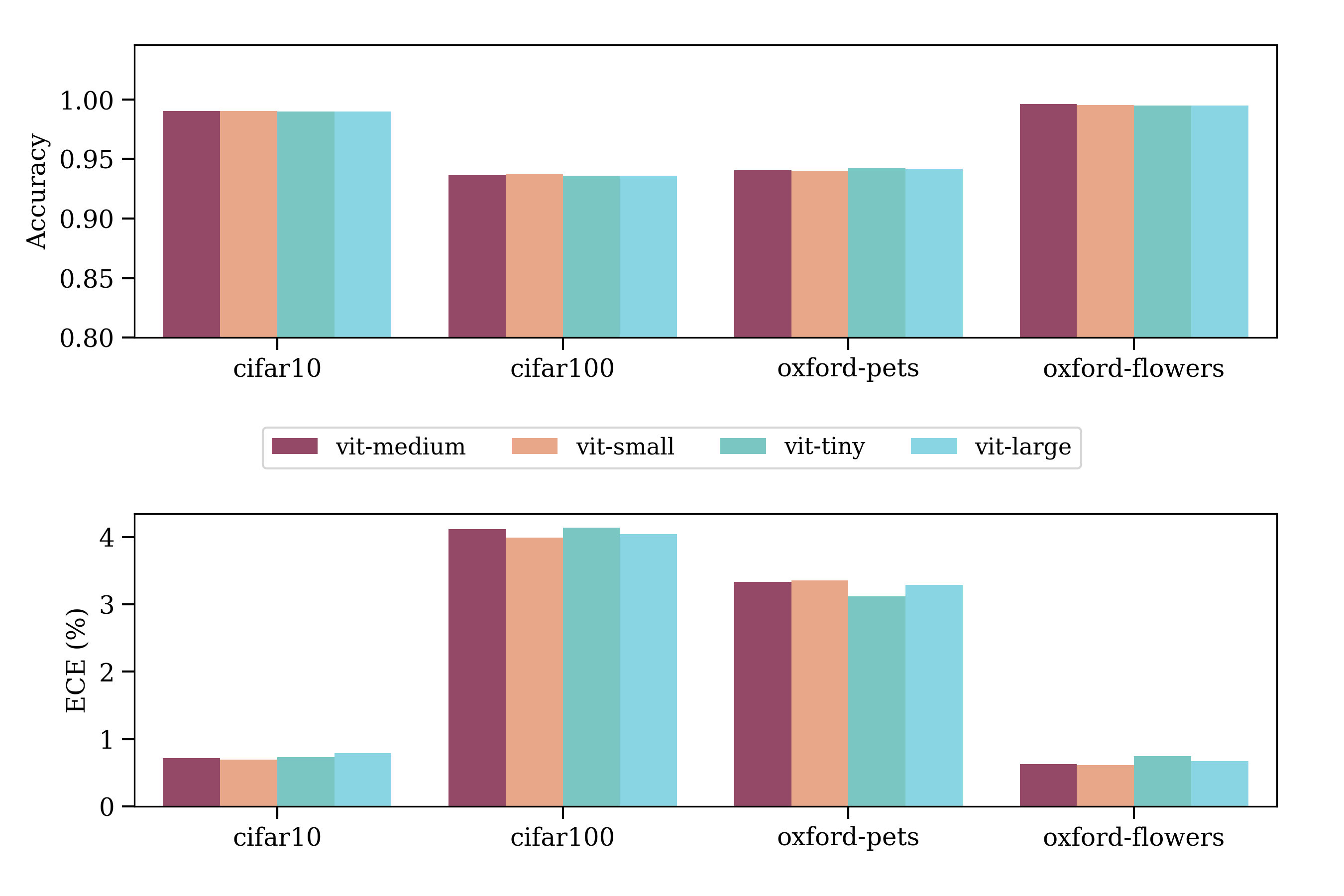}
    \caption{
    Ablation study of posterior network encoders on classification performance. We compare Vision Transformer variants: 
    \textbf{ViT-Large} (ViT-L/16 on ImageNet-21k), 
    \textbf{ViT-Medium} (DeiT-Base, distilled on ImageNet-1k), 
    \textbf{ViT-Small} (DeiT-Small, distilled on ImageNet-1k), and 
    \textbf{ViT-Tiny} (DeiT-Tiny, distilled on ImageNet-1k).
    }
    \label{fig:vision_ablation}
\end{figure}

\subsection{Additional Results on Materials Property Prediction}
Figure~\ref{fig:convergence_bandgap} shows the convergence behavior of BaLoRA on the QMOF bandgap prediction task as a function of the number of Monte Carlo (MC) forward passes at test time. The left panel reports MAE (in eV), which remains stable at approximately $0.267$~eV across all MC budgets from $10$ to $100$ steps, with narrow confidence bands indicating low variance across seeds. This confirms that the predictive mean is already well-estimated with a modest number of stochastic passes, and additional MC steps do not degrade accuracy. The right panel reports the Spearman rank correlation between predictive variance and squared error, a measure of uncertainty calibration quality. Here, a clear upward trend is observed: the correlation improves from approximately $0.315$ at $10$ MC steps to roughly $0.347$ at $100$ steps, with the most rapid gains occurring between $10$ and $50$ steps before plateauing. This indicates that while a small MC budget suffices for accurate point predictions, a moderately larger number of forward passes (${}\geq 50$) is beneficial for obtaining well-ranked uncertainty estimates. The shaded confidence intervals widen slightly at higher MC budgets, reflecting increased seed-level variability in the correlation statistic, but the overall trend is monotonically increasing. Together, these results demonstrate that BaLoRA provides stable predictions and progressively improving uncertainty quantification at test time without any retraining.
\begin{figure}[httb]
    \centering
    \includegraphics[width=0.8\linewidth]{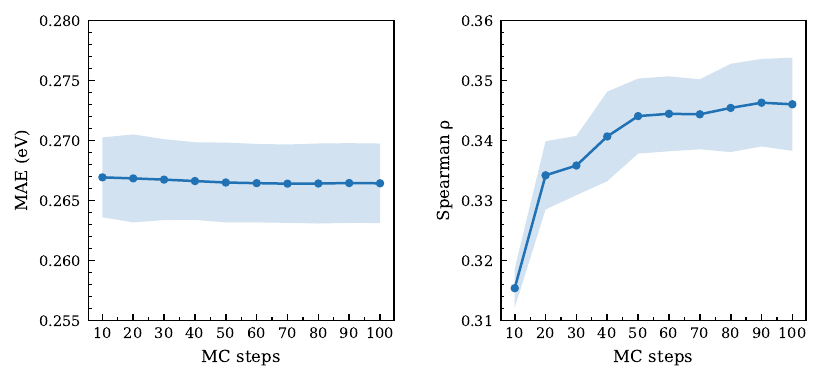}
    \caption{Performance convergence in eV for BaLoRA on bandgap prediction using \textbf{MOFTransformer} on QMOF dataset. The curve is obtained with test-time multiple forward passes (MC steps). No retraining is performed, uncertainty estimates are obtained purely through stochastic forward passes. Shaded regions denote ±1 standard deviation across three random seeds.}
    \label{fig:convergence_bandgap}
\end{figure}
%%%%%%%%%%%%%%%%%%%%%%%%%%%%%%%%%%%%%%%%%%%%%%%%%%%%%%%%%%%%

% \newpage
% \input{checklist.tex}

\end{document}